\documentclass{article}

\usepackage{hyperref}
\usepackage{url}
\usepackage{listings}
\usepackage{pifont}
\usepackage{tikz}
\usepackage{subfig}
\usepackage{multirow}
\usepackage{pgfplots}
\usepackage{booktabs}
\usepackage{cleveref}
\usepackage{dcolumn}
\usepackage{cancel}
\usepackage{ifmtarg}
\usepackage{verbatim}
\usepackage{graphicx}
\usepackage{chngpage}
\usepackage{xspace}
\usepackage{algorithm}
\usepackage{algpseudocode}
\usepackage{balance}
\usepackage{amsmath}
\usepackage{amssymb}
\usepackage{amsthm}

\usepackage{microtype}
\usepackage{booktabs} 
\usepackage[accepted]{icml2018}
\icmltitlerunning{Exploring Hidden Dimensions in Parallelizing Convolutional Neural Networks}

\newtheorem{theorem}{Theorem}
\DeclareMathOperator*{\argmin}{arg\,min}

\newcommand{\alexnet}{AlexNet\xspace}
\newcommand{\vgg}{VGG-16\xspace}
\newcommand{\inception}{Inception-v3\xspace}

\newcommand{\tensorflow}{TensorFlow\xspace}
\newcommand{\pytorch}{PyTorch\xspace}
\newcommand{\caffe}{Caffe2\xspace}
\newcommand{\Sys}{DeePa\xspace}

\newcommand{\er}[1]{\mbox{\rm\em #1}}
\newif\ifdeadcode

\begin{document}

\twocolumn[
\icmltitle{Exploring Hidden Dimensions in Parallelizing Convolutional Neural Networks}

\begin{icmlauthorlist}
\icmlauthor{Zhihao Jia}{S}
\icmlauthor{Sina Lin}{M}
\icmlauthor{Charles R. Qi}{S}
\icmlauthor{Alex Aiken}{S}
\end{icmlauthorlist}

\icmlaffiliation{S}{Stanford University}
\icmlaffiliation{M}{Microsoft}

\icmlcorrespondingauthor{Zhihao Jia}{zhihao@cs.stanford.edu}

\icmlkeywords{Parallel Convolutional Neural Networks}
\vskip 0.3in
]

\printAffiliationsAndNotice{}

\setcounter{footnote}{1}

\begin{abstract}


The past few years have witnessed growth in the computational requirements for training deep convolutional neural networks. 
Current approaches parallelize training onto multiple devices by applying a single parallelization strategy (e.g., data or model parallelism) to all layers in a network. 
Although easy to reason about, these approaches result in suboptimal runtime performance in large-scale distributed training, since different layers in a network may prefer different parallelization strategies.
In this paper, we propose {\em layer-wise parallelism} that allows each layer in a network to use an individual parallelization strategy.
We jointly optimize how each layer is parallelized by solving a graph search problem. 
Our evaluation shows that layer-wise parallelism
outperforms state-of-the-art approaches by increasing training throughput, reducing communication costs, achieving better scalability to multiple GPUs,
while maintaining original network accuracy.


\end{abstract}

\section{Introduction}
\label{sec:intro}

Convolutional neural networks (CNNs) have proven to be general and effective across many tasks including image classification~\cite{AlexNet, Inception}, face recognition~\cite{face}, 
text classification~\cite{text}, and game playing~\cite{game}. 
Their success has resulted in growth in the computational requirements to train today's CNNs, which takes days or even weeks on modern processors~\cite{VUCN, VGG, Inception}.

Previous work has investigated parallelization techniques to accelerate training.
The most common approach is {\em data parallelism}~\cite{AlexNet, VGG}
that keeps a replica of an entire network on each device and assigns a subset of the training data to each device.
Another common approach is {\em model parallelism}~\cite{DevicePlace, SplitNet} that divides the network parameters into disjoint subsets and trains each subset on a dedicated device.
Both approaches apply a single parallelization strategy (i.e., data or model parallelism) to all layers in a CNN. 
However, within a CNN, different layers may prefer different parallelization strategies for achieving optimal runtime performance.
For example, a densely-connected layer with millions of parameters prefers model parallelism to reduce communication cost for synchronizing parameters,
while a convolutional layer typically prefers data parallelism to eliminate data transfers from the previous layer.
In addition, some layers may prefer more sophisticated parallelization strategies such as parallelizing in a mixture of multiple data dimensions (see Section~\ref{sec:hidden}). 
Because of the differing characteristics of different layers in a network, applying a single parallelization strategy to all layers usually results in suboptimal runtime performance.


In this paper, we propose {\em layer-wise parallelism}, which enables each layer in a network to use an individual parallelization strategy.
Layer-wise parallelism performs the same computation for each layer as it is defined in the original network and therefore maintains the same network accuracy by design.
Compared to existing parallelization approaches, our approach defines a more comprehensive search space of parallelization strategies, which includes data and model parallelism as two special cases.
Our goal is to find the parallelization strategies for individual layers to jointly achieve the best possible runtime performance while maintaining the original network accuracy.
To formalize the problem, we introduce {\em parallelization configurations} that define the search space for parallelizing a layer across multiple devices.
We propose a cost model that quantitively evaluates the runtime performance of different parallelization strategies. 
The cost model considers both the computation power of each device and the communication bandwidth between devices.
With the cost model, we convert the original problem of choosing parallelization configurations for individual layers to a graph search problem and 
develop an efficient algorithm to find a globally optimal strategy under the cost model.

We evaluate the runtime performance of layer-wise parallelism with \alexnet~\cite{AlexNet}, \vgg~\cite{VGG}, and \inception~\cite{Inception} on the ILSVRC 2012 image classification dataset.
For distributed training on 16 P100 GPUs (on 4 nodes), layer-wise parallelism is 1.4-2.2$\times$ faster than state-of-the-art parallelization strategies.
Note that the speedup is achieved without sacrificing network accuracy, since layer-wise parallelism trains the same network as data and model parallelism and uses more efficient parallelization strategies to achieve better runtime performance.
In addition, layer-wise parallelism reduces communication costs by 1.3-23.0$\times$ compared to data and model parallelism. 
Finally, we show that layer-wise parallelism achieves better scalability than other parallelization strategies. Scaling the training of \inception from 1 GPU to 16 GPUs, layer-wise parallelism obtains 15.5$\times$ speedup, while other parallelization strategies achieve at most 11.2$\times$ speedup.

To summarize, our contributions are:
\begin{itemize}
\item We propose layer-wise parallelism, which allows different layers in a network to use individual parallelization configurations. 
\item We define the search space of possible parallelization configurations for a layer and present a cost model to quantitively evaluate the runtime performance of training a network.
Based on the cost model, we develop an efficient algorithm to jointly find a globally optimal parallelization strategy.
\item We provide an implementation that supports layer-wise parallelism and show that layer-wise parallelism can increase training throughput by 1.4-2.2$\times$ and reduce communication costs by 1.3-23.0$\times$ over state-of-the-art approaches while improving scalability.
\end{itemize}

\section{Related Work}
\label{sec:related}

{\bf Data and model parallelism} have been widely used by existing deep learning frameworks (e.g., \tensorflow~\cite{Tensorflow}, \caffe\footnote{https://caffe2.ai}, and \pytorch\footnote{https://pytorch.org}) to parallelize training.
Data parallelism~\cite{AlexNet} keeps a copy of an entire network on each device, which is inefficient for layers with large numbers of network parameters and becomes a scalability bottleneck in large scale distributed training.
Model parallelism~\cite{DistBelief} divides network parameters into disjoint subsets and trains each subset on a dedicated device. 
This reduces communication costs for synchronizing network parameters but exposes limited parallelism.


\citet{OWT} introduces ``one weird trick'' (OWT) that uses data parallelism for convolutional and pooling layers and switches to model parallelism for fully-connected layers to accelerate training.
This achieves better runtime performance than data and model parallelism but is still suboptimal.
In this paper, we use OWT parallelism as a baseline in the experiments and show that
layer-wise parallelism can further reduce communication costs and improve training performance compared to OWT parallelism.

{\bf System optimizations.} A number of system-level optimizations have been proposed to accelerate large scale training.
\citet{LargeSGD} 
uses a three-step allreduce operation to optimize communication across devices and aggressively overlaps gradient synchronization with back propagation.
\citet{Poseidon} introduces a hybrid communication scheme to reduce communication costs for gradient synchronization.
All these systems are based on data parallelism and are limited in runtime performance by communication costs.

{\bf Network parameter reduction.} 
\citet{Han1} presents an iterative weight pruning method that repeatedly retrains the network while removing weak connections.
\citet{Alvarez1} proposes a network that learns the redundant parameters in each layer and iteratively eliminates the redundant parameters.
These approaches improve runtime performance by significantly reducing the number of parameters in a neural network,
which results in a modified network and may decrease the network accuracy (as reported in these papers). By contrast, in this paper, we introduce a new
approach that accelerates distributed training while maintaining the original network accuracy.

\section{Hidden Dimensions in Parallelizing a Layer}
\label{sec:hidden}
Data parallelism parallelizes training by partitioning a training dataset in the sample dimension.
However, other dimensions can also be used to parallelize a layer. 
For example, in standard CNNs for 2D images,
data is commonly organized as 4-dimensional tensors (i.e., sample, height, width, and channel). 
The {\em sample} dimension includes an index for each image in a training dataset.
The {\em height} and {\em width} dimensions specify a position in an image.
For a particular position, the {\em channel} dimension
indexes different neurons for that position. 

In principle, any combination of these dimensions can be used to parallelize a layer,
and we should consider both selecting the dimensions to parallelize training and the degree of parallelism in each dimension.
Exploring these additional dimensions has the following advantages.

\begin{figure}[t]
\centering
\includegraphics[scale=0.3]{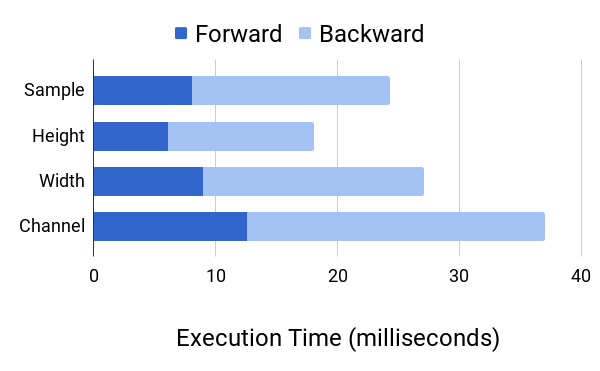}
\vspace{-6mm}
\caption{Execution time for parallelizing a convolutional layer (Conv8 in \vgg~\cite{VGG}) on 4 GPUs by using different dimensions.
Parallelizing a layer in other dimensions preserves the same output as parallelizing it in the sample dimension. To achieve this, different GPUs may share some common
input data for parallelizations in the height, width, and channel dimensions.}
\label{fig:conv_tp}
\vspace{-3mm}
\end{figure}

First, parallelizing a layer in other dimensions can reduce execution time. 
Figure~\ref{fig:conv_tp} shows the time to process a 2D convolutional layer on 4 GPUs using parallelism in different dimensions.
For this layer, data parallelism achieves suboptimal performance.


\begin{figure}[t]
\begin{center}
\subfloat[Parallelism in the sample \newline dimension.] {
\label{fig:data_transfer_dp}
\includegraphics[scale=0.17]{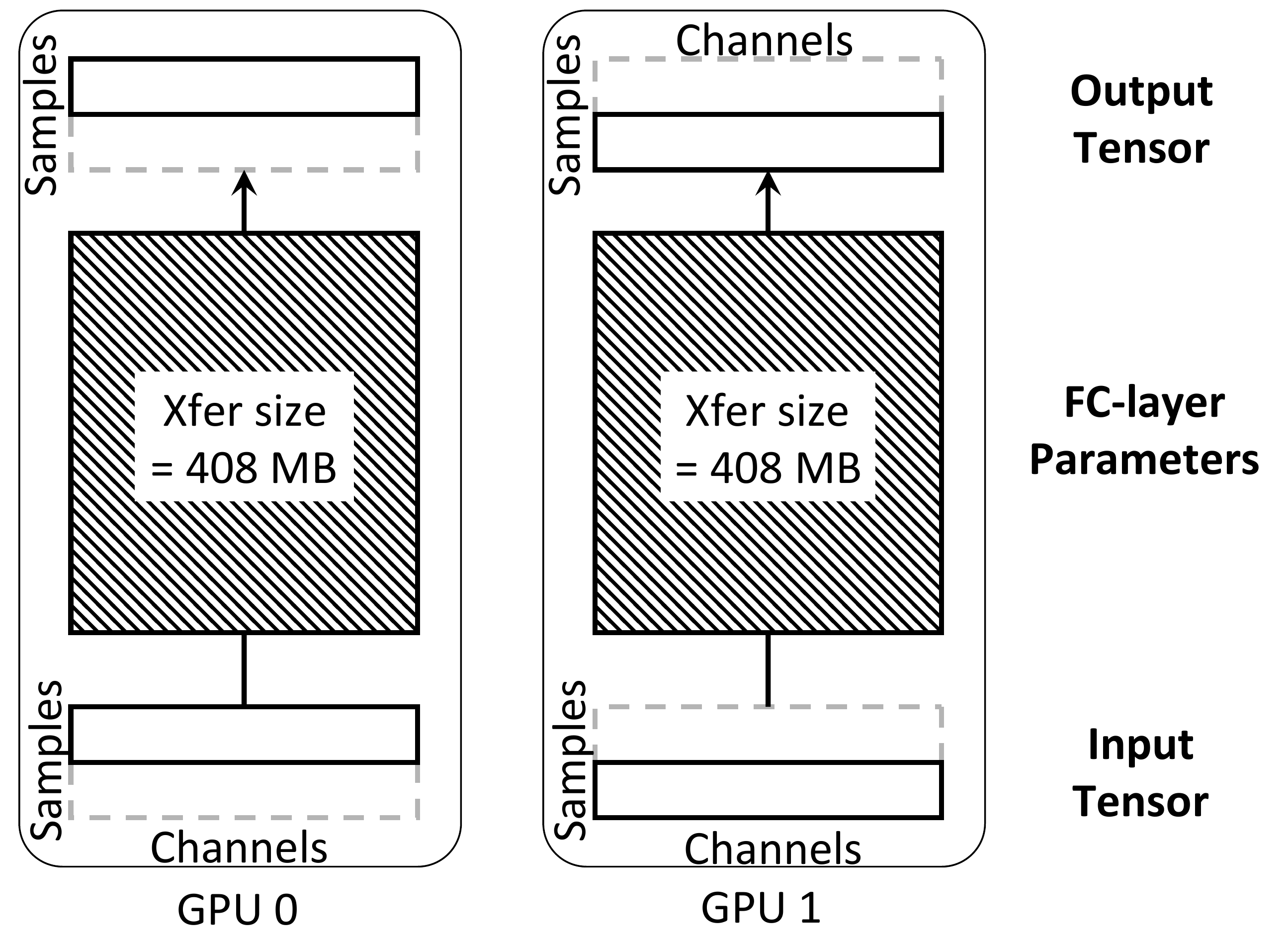}
}
\subfloat[Parallelism in the channel \newline dimension.] {
\label{fig:data_transfer_mp}
\includegraphics[scale=0.17]{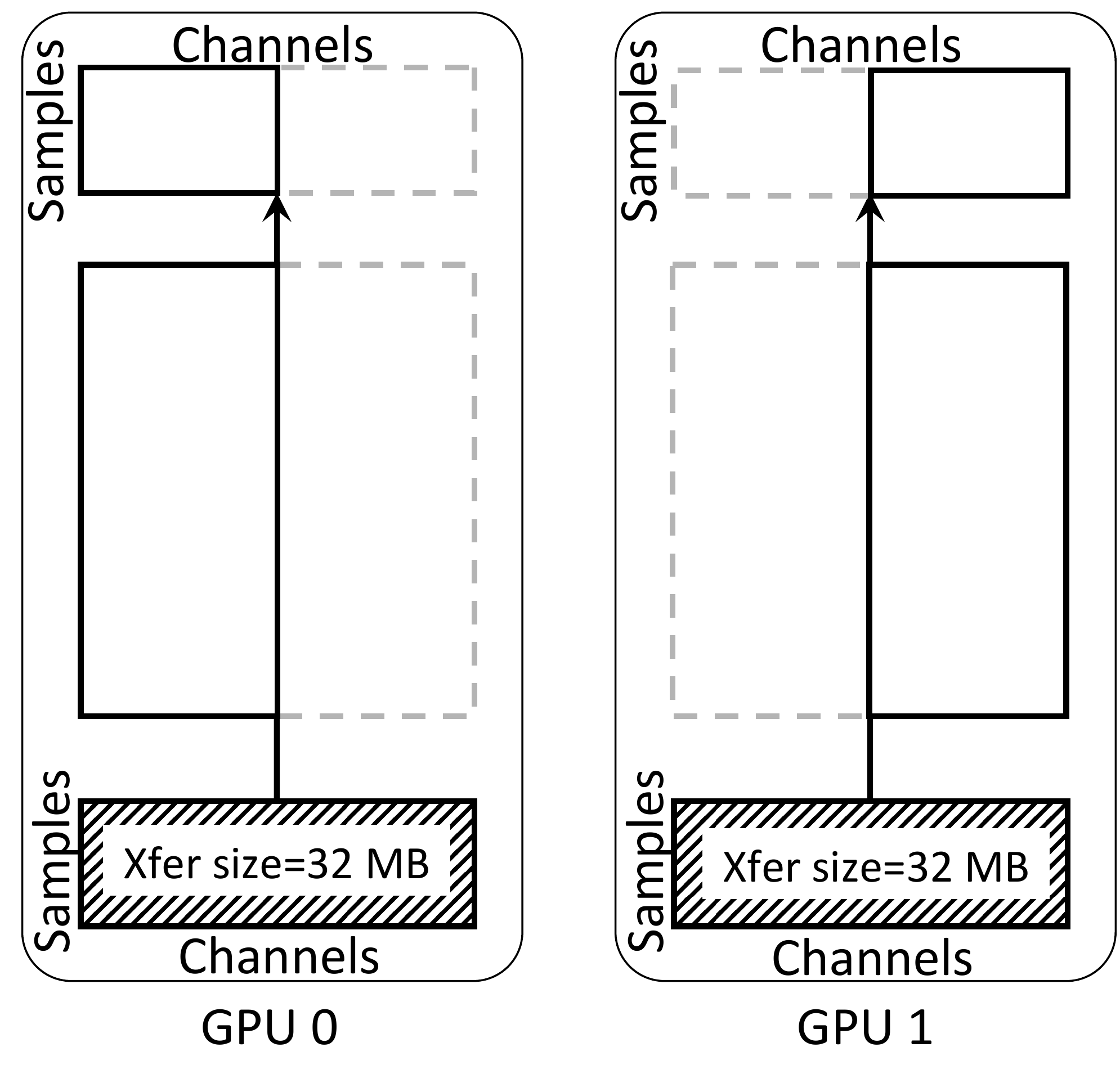}
}
\end{center}
\vspace{-3mm}
\caption{Different ways to parallelize the first fully-connected layer of \vgg. Rectangles with solid lines indicate tensors managed by the
local GPU, while rectangles with dotted lines are tensors managed by a remote GPU. The shadow rectangles indicate data transfers in each step.}
\label{fig:data_transfer}
\vspace{-3mm}
\end{figure}

Second, exploring parallelism in other dimensions can reduce communication costs.
Figure~\ref{fig:data_transfer} shows an example of parallelizing a fully-connected layer on two GPUs in different dimensions.
In data parallelism (Figure~\ref{fig:data_transfer_dp}), each GPU synchronizes the gradients of the entire fully-connected layer (shown as the shadow rectangles) in every step. 
An alternative approach (Figure~\ref{fig:data_transfer_mp}) parallelizes in the channel dimension, which eliminates parameter synchronization, as different GPUs train disjoint subsets of the parameters, but introduces
additional data transfers for input tensors (shown as the shadow rectangles).
For this particular case, using parallelism in the channel dimension reduces communication costs by 12$\times$.

\begin{figure}[t]
\vspace{-3mm}
\centering
\includegraphics[scale=0.28]{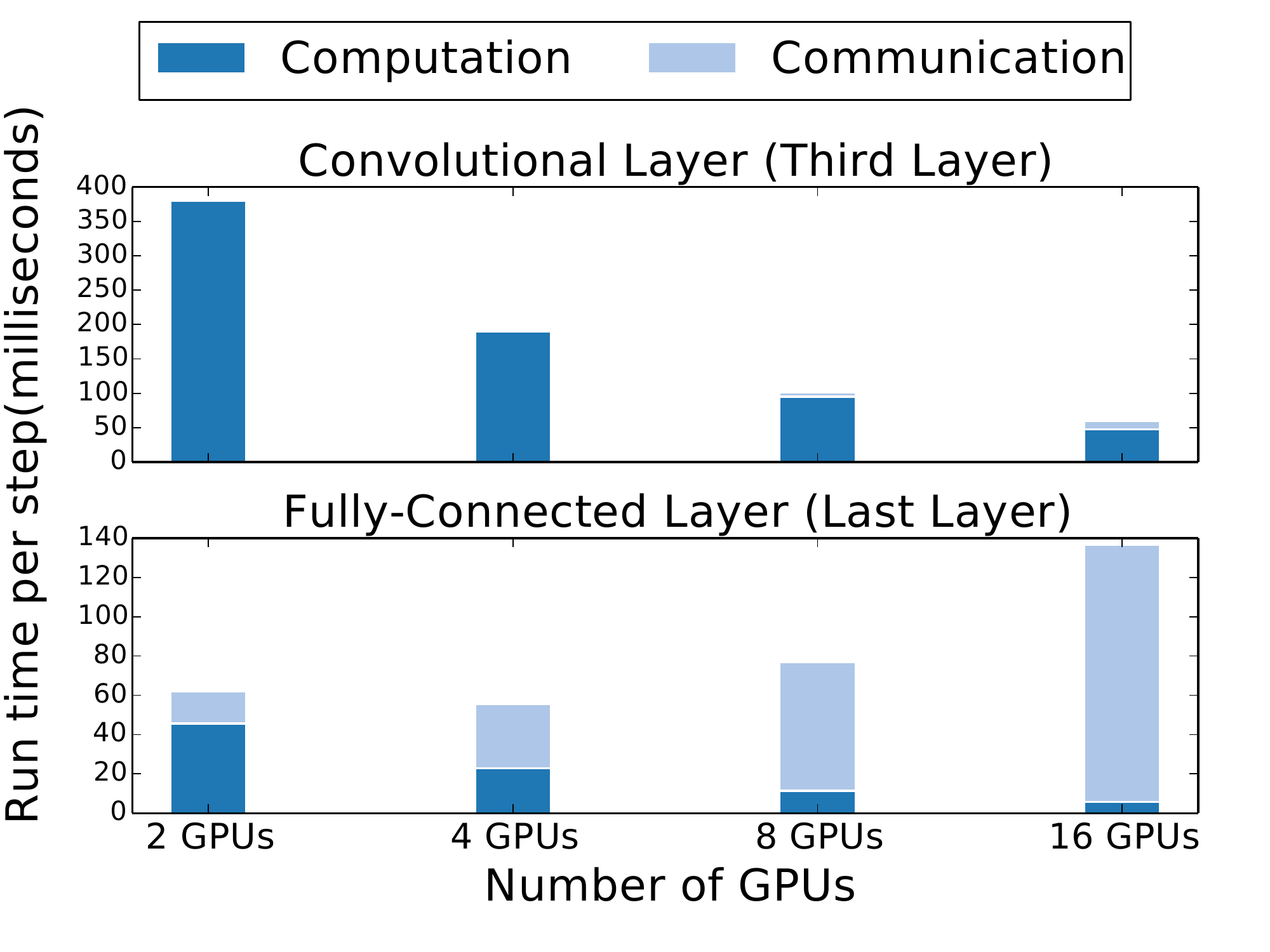}
\vspace{-6mm}
\caption{Computation and communication time to process the third layer and the last layer of \inception using data parallelism.}
\vspace{-3mm}
\label{fig:balance}
\end{figure}

Third, the {\em degree} of parallelism (number of parallel devices) is another dimension that affects runtime performance.
Different layers have different execution time and communication costs and may prefer different degrees of parallelism.
Figure~\ref{fig:balance} shows the runtime performance of processing two layers in \inception with different degrees of parallelism.
The convolutional layer performs best on 16 GPUs, while the fully-connected layer performs best on 4 GPUs. 


\section{Problem Definition}
\label{sec:problem}
We define the parallelization problem with two graphs. The first is a {\em device graph} that models all
available hardware devices and the connections between them. The second is
a {\em computation graph} that defines the neural network to be mapped onto the device graph. 

In a device graph $\mathcal{D}$, each node $d_i$ is a device (e.g., a CPU or a GPU),
and each edge $(d_i, d_j)$ is an connection between $d_i$ and $d_j$ with communication bandwidth $b(d_i, d_j)$.
In a computation graph $\mathcal{G}$, each node $l_i \in \mathcal{G}$ is a layer in the neural network,
and each edge $(l_i, l_j) \in \mathcal{G}$ is a tensor that is an output of layer $l_i$ and an input of layer $l_j$.

\begin{table}
\vspace{-2mm}
\caption{Parallelizable dimensions for different layers. The {\em length} dimension specifies a position in a 1D image.}
\label{tab:para_dimensions}
\centering
\resizebox{\columnwidth}{!}{
\begin{tabular}{|c|c|}
\hline
{\bf Layer} & {\bf Parallelizable dimensions} \\
\hline
Fully-connected & \{sample, channel\} \\
1D convolution/pooling & \{sample, channel, length\} \\
2D convolution/pooling & \{sample, channel, height, width\} \\
3D convolution/pooling & \{sample, channel, height, width, depth\} \\
\hline
\end{tabular}
}
\vspace{-4mm}
\end{table}
We now define the parallelization of a layer. To parallelize a layer across multiple devices, we assume that different devices can process the layer
in parallel without any dependencies. This requires different devices to compute {\em disjoint} subsets of a layer's output tensor.
Therefore, we describe the parallelization of a layer by defining how its output tensor is partitioned.

For a layer $l_i$, we define its {\em parallelizable dimensions} $\mathcal{P}_i$ as the set of all divisible dimensions in its
output tensor. $\mathcal{P}_i$ includes all dimensions to parallelize the layer $l_i$. 
Table~\ref{tab:para_dimensions} shows the parallelizable dimensions for different layers.

\begin{figure}[t]
\vspace{-2mm}
\begin{center}
\subfloat[($n$=1, $c$=1, $h$=1, $w$=4)] {
\label{fig:config1}
\includegraphics[scale=0.26]{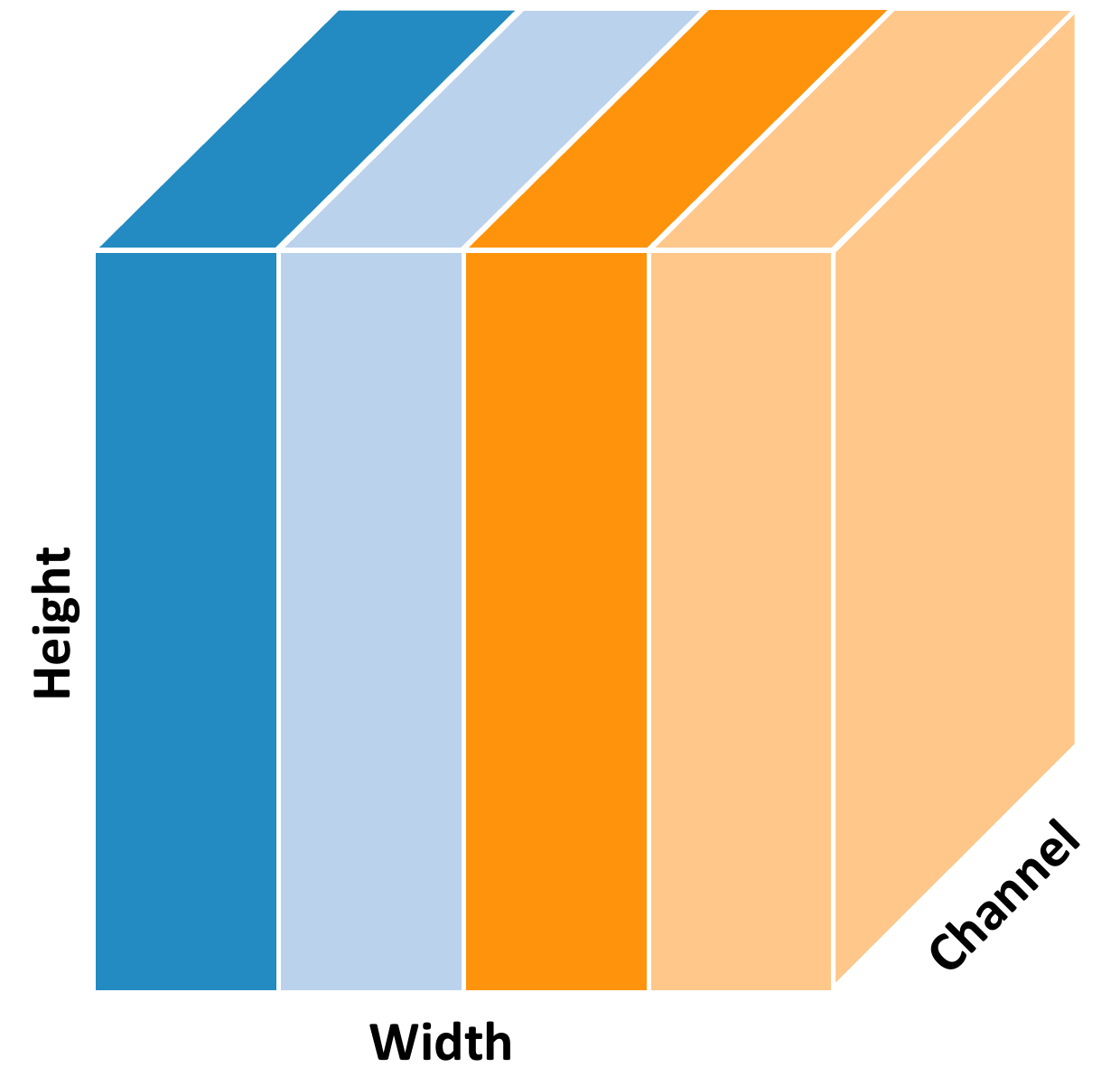}
}
\subfloat[($n$=1, $c$=1, $h$=4, $w$=1)] {
\label{fig:config2}
\includegraphics[scale=0.26]{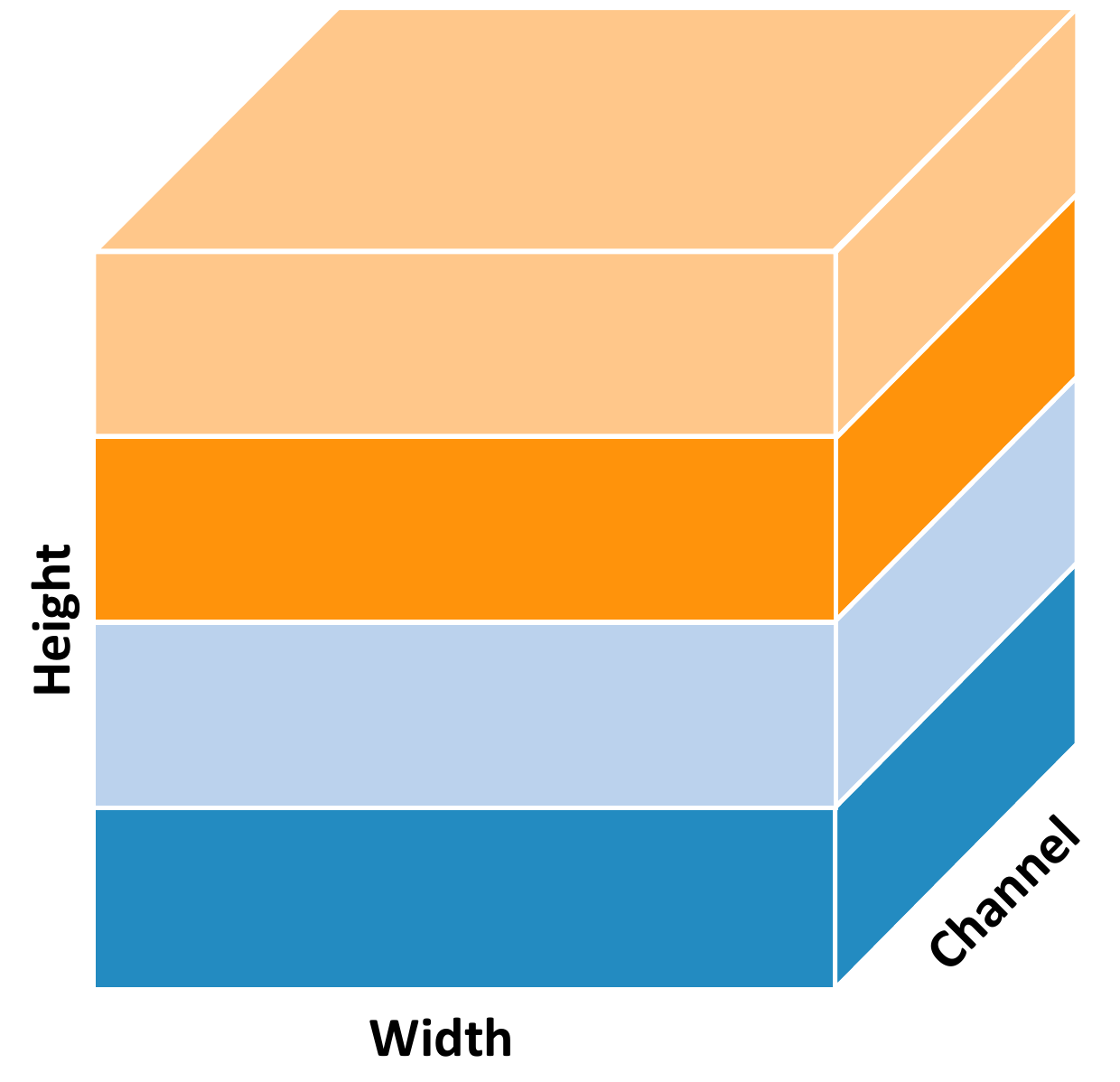}
}
\\
\vspace{-2mm}
\subfloat[($n$=1, $c$=4, $h$=1, $w$=1)] {
\label{fig:config3}
\includegraphics[scale=0.26]{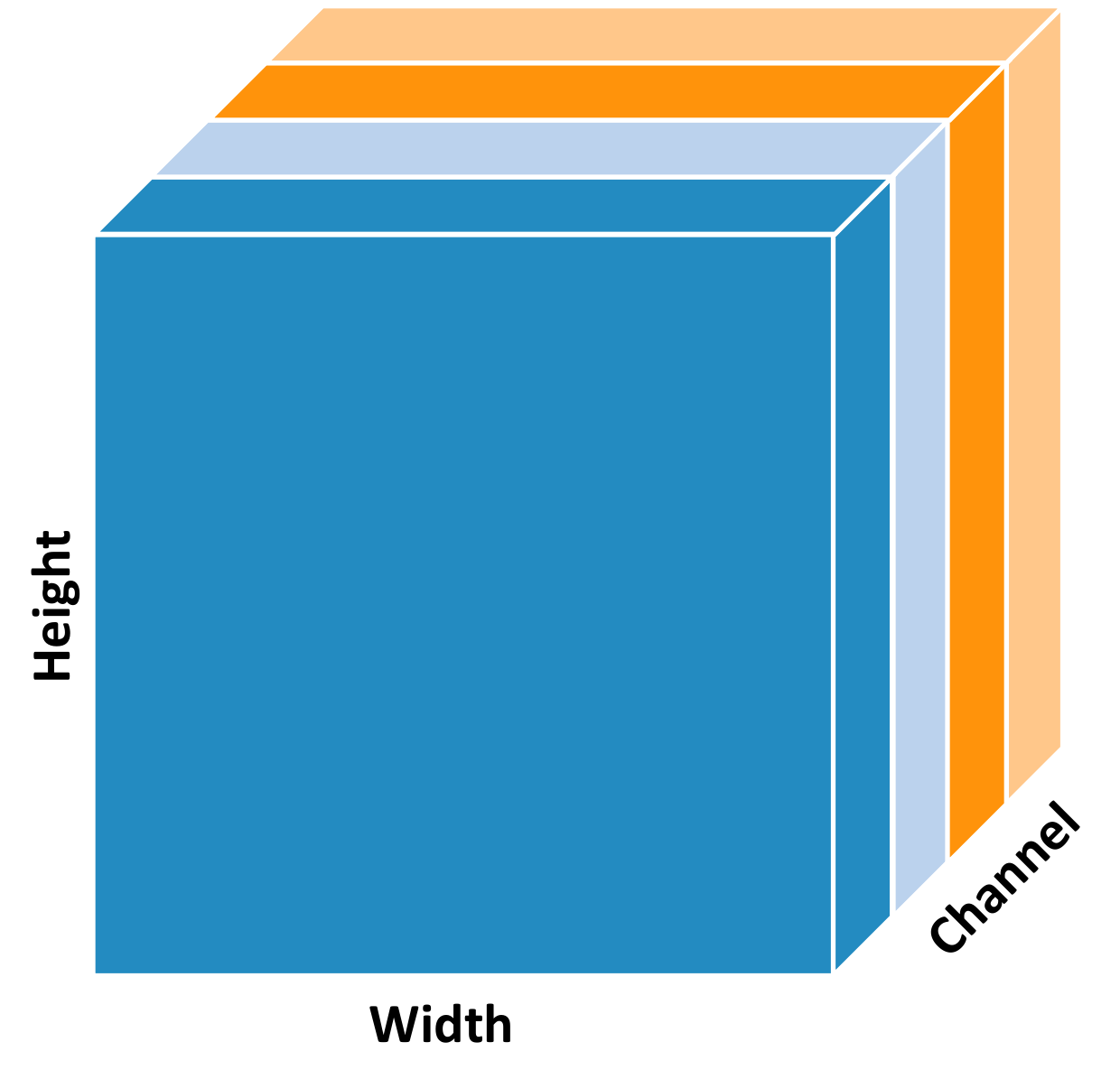}
}
\subfloat[($n$=1, $c$=1, $h$=2, $w$=2)] {
\label{fig:config4}
\includegraphics[scale=0.26]{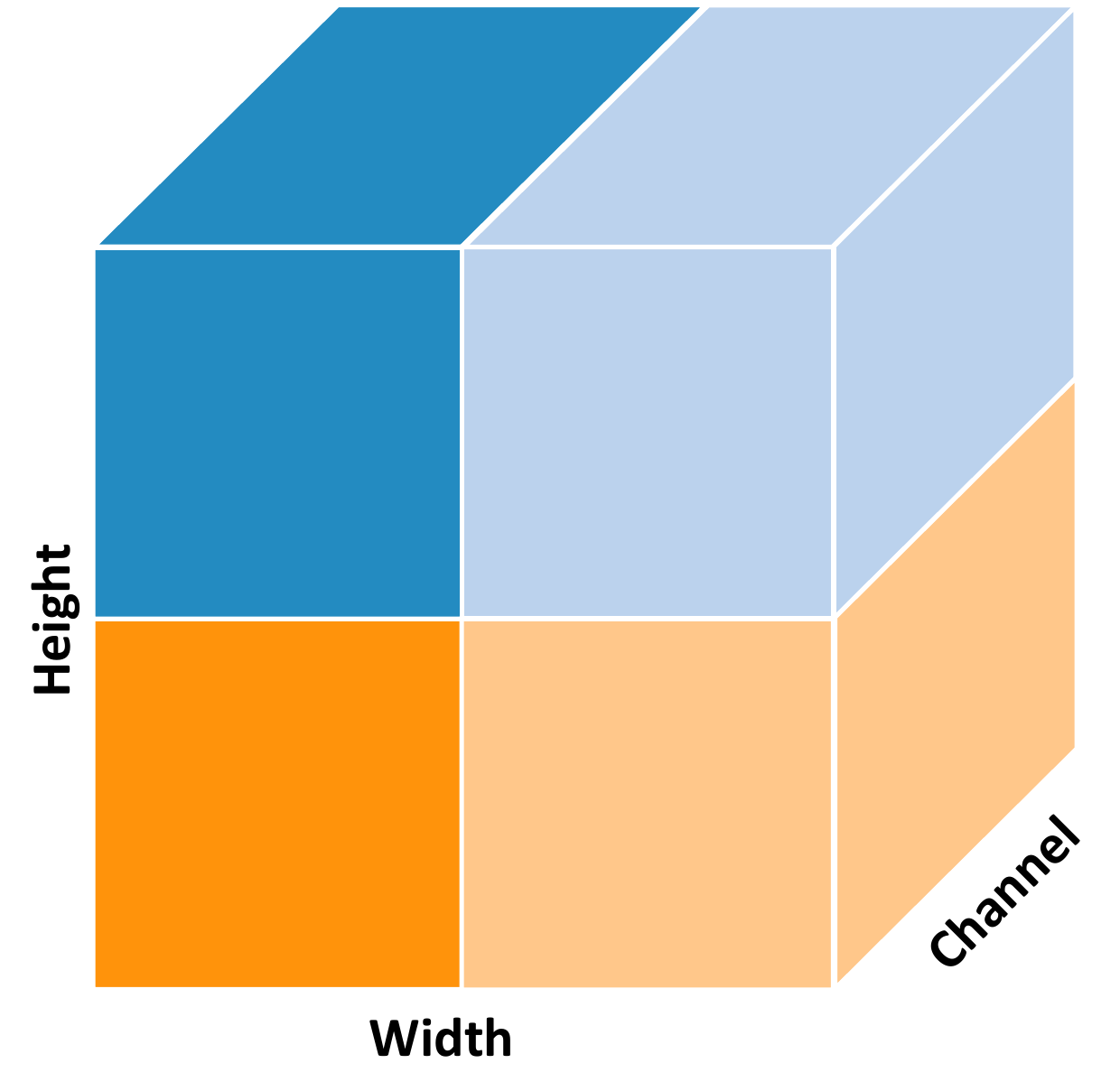}
}
\vspace{-2mm}
\caption{Example configurations to parallelize a 2D convolutional layer in a single dimension or combinations of
multiple dimensions. The figure shows how each training sample is partitioned.}
\label{fig:config_examples}
\end{center}
\vspace{-2mm}
\end{figure}

A {\em parallelization configuration} $c_i$ of a layer $l_i$ defines how $l_i$ is parallelized across different devices.
For each parallelizable dimension in $\mathcal{P}_i$, $c_i$ includes a positive integer that describes the degree of parallelism in that dimension.
For a configuration $c_i$, the product of the integers over all dimensions is the total degree of parallelism for $l_i$.
We assume equal partitioning in each parallelizable dimension, which provides well-balanced workload among multiple devices.
Figure~\ref{fig:config_examples} demonstrates some possible configurations for parallelizing a 2D convolutional layer over four devices.
Parallelizing a layer in any configuration produces the same output. This guarantees that all
configurations parallelize training on the original network and therefore maintains the original network accuracy.

A {\em parallelization strategy} $\mathcal{S}$ includes a configuration $c_i$ for each layer $l_i \in \mathcal{G}$.
Let $t(\mathcal{G}, \mathcal{D}, \mathcal{S})$ denote the per-iteration execution time to parallelize the computation graph $\mathcal{G}$ on the device graph $\mathcal{D}$ by using strategy $\mathcal{S}$.
Our goal is to find a parallelization strategy $\mathcal{S}$ such that the per-iteration training time $t(\mathcal{G}, \mathcal{D}, \mathcal{S})$ is minimized.
 
\section{Method}
\label{sec:method}

\subsection{Cost Model}
We introduce a cost model to quantitively evaluate the runtime performance of different parallelization strategies 
and use an {\em dynamic programming based} graph search algorithm to find an
optimal parallelization strategy under our cost model. The cost model depends on the following assumptions:
\begin{itemize}
\item[1.] For a layer $l_i \in \mathcal{G}$, the time to process $l_i$ is predictable with low variance and is largely independent of the contents of the input data.
\item[2.] For each connection $(d_i, d_j)$ between device $d_i$ and $d_j$ with bandwidth $b$, transferring a tensor of size $s$ from 
$d_i$ to $d_j$ takes $s / b$ time (i.e., the communication bandwidth can be fully utilized).
\item[3.] The runtime system has negligible overhead. A device begins processing a layer as soon as its input tensors are
available and the device has finished previous tasks.
\end{itemize}

Most layers in CNNs are based on dense matrix operations, whose execution time satisfies the first assumption.
In addition, the experiments show that our implementation satisfies the second and third assumptions well enough to obtain significant runtime performance improvements.

We define three cost functions on computation graphs:
\begin{itemize}
\item[1.] For each layer $l_i$ and its parallelization configuration $c_i$, $t_{\er c}(l_i, c_i)$ is the time to process the layer $l_i$ under configuration $c_i$. 
This includes both the forward and back propagation time and 
is estimated by processing the layer under that configuration multiple times on the device and measuring the average execution time.

\item[2.] For each tensor $e=(l_i, l_j)$, $t_{\er x}(e, c_i, c_j)$ estimates the time to transfer the input tensors to the target devices,
using the size of the data to be moved and the known communication bandwidth.


\item[3.] For each layer $l_i$ and its parallelization configuration $c_i$, $t_{\er s}(l_i, c_i)$ is the time to synchronize the parameters in layer $l_i$ after back propagation. 
To complete parameter synchronization, each device that holds a copy of the parameters for layer $l_i$ transfers its local gradients to a parameter server that stores the up-to-date
parameters for layer $l_i$. After receiving the gradients for layer $l_i$, the parameter server applies the gradients to the parameters and transfers the updated parameters back to the device. In this process, the communication time is much longer than the execution time to update parameters, therefore we use the communication time to approximate the parameter synchronization time.
\end{itemize}

Using the three cost functions above, we define
\begin{equation}
\begin{split}
\label{eqn1}
t_{\er o}(\mathcal{G}, \mathcal{D}, \mathcal{S}) & = \sum_{l_i \in \mathcal{G}}\{t_{\er c}(l_i, c_i) + t_{\er s}(l_i, c_i)\} \\
& + \sum_{e=(l_i, l_j)\in \mathcal{G}}{t_{\er x}(e, c_i, c_j)}
\end{split}
\end{equation}

$t_{\er o}(\mathcal{G}, \mathcal{D}, \mathcal{S})$ estimates the per-step execution time for parallelization strategy $\mathcal{S}$, which includes forward
processing, back propagation, and parameter synchronization. 

\subsection{Graph Search}
\begin{figure}[t]
\vspace{-3mm}
\begin{center}
\subfloat[Node elimination.]{
\includegraphics[scale=0.35]{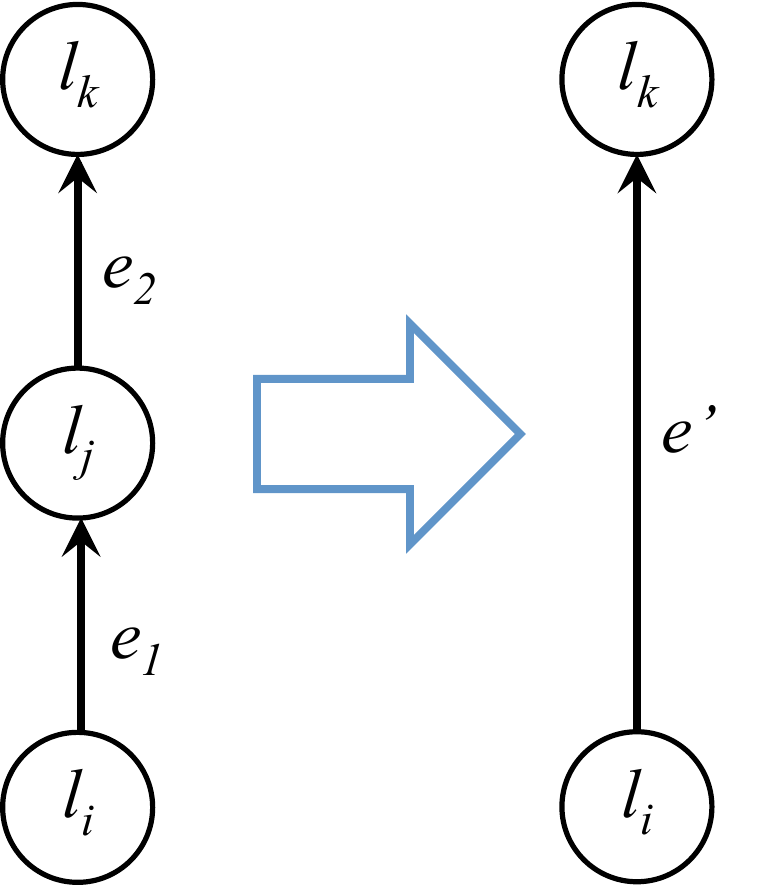}
\label{fig:reduce_node}
}
\subfloat[Edge elimination.]{
\includegraphics[scale=0.35]{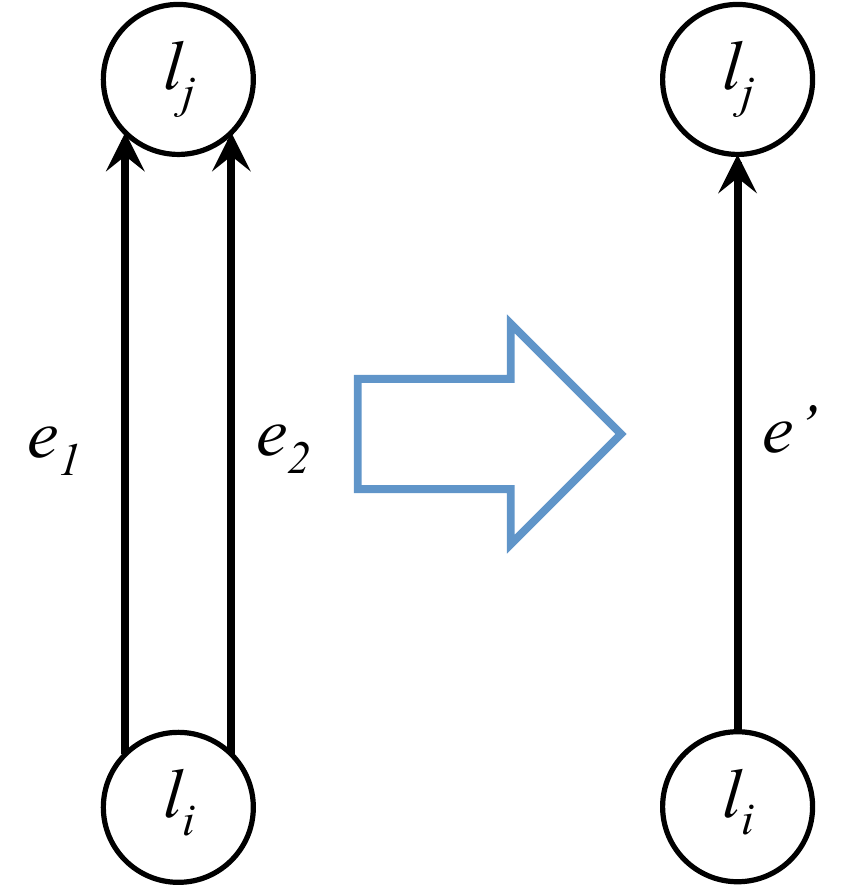}
\label{fig:reduce_edge}
}
\caption{Node and edge elimination on a computation graph.}
\end{center}
\vspace{-3mm}
\end{figure}

Equation~\ref{eqn1} expresses the problem of finding an optimal parallelization strategy as a graph search problem: our goal is to find a strategy $\mathcal{S}$ so that the overall runtime cost $t_{\er o}(\mathcal{G}, \mathcal{D}, \mathcal{S})$ is minimized.


Since layer-wise parallelism allows each layer to use an individual configuration, 
the number of potential parallelization strategies is exponential in the number of layers in a computation graph, which makes it impractical
to enumerate all strategies for large CNNs.  However, the CNNs we have seen in practice exhibit strong locality: each layer is only connected to a few
layers with similar depths in a computation graph. Based on this observation, we use the following two graph reductions
to iteratively simplify a computation graph while preserving optimal parallelization strategies.

{\bf Node elimination.} If a computation graph $\mathcal{G}$ includes a node $l_j$ with a single in-edge $e_1 = (l_i, l_j)$ and a single
out-edge $e_2=(l_j, l_k)$, a node elimination removes node $l_j$ and the two edges $e_1$ and $e_2$ from $\mathcal{G}$, inserts a new edge
$e'=(l_i, l_k)$ back into $\mathcal{G}$, and returns the modified graph (Figure~\ref{fig:reduce_node}).
We define $t_{\er x}(e', \cdot, \cdot)$ in a way that preserves optimal parallelization strategies (see Theorem~\ref{thm1}). 
\begin{equation}
\begin{split}
\label{eqn2}
t_{\er x}(e', c_i, c_k) = \min_{c_j} & \{ t_{\er c}(l_j, c_j) + t_{\er s}(l_j, c_j) \\
& + t_{\er x}(e_1, c_i, c_j) + t_{\er x}(e_2, c_j, c_k)\}
\end{split}
\end{equation}

Intuitively, we use dynamic programming to compute an optimal configuration $c_j$ for node $l_j$ for every possible combination of $c_i$ and $c_k$
and use the cost functions associated with $l_j$ to define $t_{\er x}(e', c_i, c_k)$.
\begin{theorem}
\label{thm1}
Assume $\mathcal{G}'$ = NodeElimination($\mathcal{G}$) and $l_j$ is the eliminated node. If ${\mathcal{S}_o}'$ is an optimal parallelization srategy for $\mathcal{G}'$,
then $\mathcal{S}_o = {\mathcal{S}_o}' + c_j$ is an optimal parallelization strategy for $\mathcal{G}$, where $c_j$ minimizes Equation~\ref{eqn2}.
\end{theorem}

{\bf Edge elimination.} If a computation graph $\mathcal{G}$ includes two edges with the same source and destination nodes
 (i.e., $e_1 = (l_i, l_j)$ and $e_2 = (l_i, l_j)$), an edge elimination removes $e_1$ and $e_2$ from $\mathcal{G}$, inserts a new edge $e'=(l_i, l_j)$ into $\mathcal{G}$ (Figure~\ref{fig:reduce_edge}).
 We define $t_{\er x}(e', \cdot, \cdot)$ using $t_{\er x}(e_1, \cdot, \cdot)$ and $t_{\er x}(e_2, \cdot, \cdot)$.
\begin{equation}
\label{eqn3}
t_{\er x}(e', c_i, c_j) = t_{\er x}(e_1, c_i, c_j) + t_{\er x}(e_2, c_i, c_j)
\end{equation}
\begin{theorem}
\label{thm2}
Assume $\mathcal{G}'$ = EdgeElimination($\mathcal{G}$), and ${\mathcal{S}_o}'$ is an optimal parallelization strategy of $\mathcal{G}'$,
then ${\mathcal{S}_o} = {\mathcal{S}_o}'$ is an optimal parallelization strategy of $\mathcal{G}$.
\end{theorem}

\begin{algorithm}[t]
\caption{Finding Optimal Parallelization Strategy $\mathcal{S}$.}
\begin{algorithmic}[1]
\State {\bf Input:} A computation graph $\mathcal{G}$, a device graph $\mathcal{D}$, and precomputed cost functions (i.e., $t_{\er c} (\cdot)$, $t_{\er s} (\cdot)$ and $t_{\er x}(\cdot)$ )
\State {\bf Output:} A parallelization strategy $\mathcal{S}$ minimizing $t_{\er o}(\mathcal{G}, \mathcal{D}, \mathcal{S})$
\State
\State $\mathcal{G}^{(0)} = \mathcal{G}$
\State $m = 0$
\While{true}
\State $\mathcal{G}^{(m+1)}$ = \Call{NodeElimination}{$\mathcal{G}^{(m)}$}
\State $\mathcal{G}^{(m+2)}$ = \Call{EdgeElimination}{$\mathcal{G}^{(m+1)}$}
\If {$\mathcal{G}^{(m+2)} = \mathcal{G}^{(m)}$}
\State {\bf break}
\EndIf
\State $m=m+2$
\EndWhile
\State Find the optimal strategy $\mathcal{S}^{(m)}$ for $\mathcal{G}^{(m)}$ by enumerating all possible candidate strategies
\For {$i = m$-1 to $0$}
\If {$\mathcal{G}^{(i+1)}$ = \Call{NodeElimination}{$\mathcal{G}^{(i)}$}}
\State \Comment{Assume $l_j$ is the node eliminated from $\mathcal{G}^{(i)}$}
\State Find $c_j$ that minimizes Equation~\ref{eqn1}
\State $\mathcal{S}^{(i)} = \mathcal{S}^{(i+1)} + c_j$
\Else
\State $\mathcal{S}^{(i)} = \mathcal{S}^{(i+1)}$
\EndIf
\EndFor
\State {\bf return} $\mathcal{S}^{(0)}$
\end{algorithmic}
\label{alg1}
\end{algorithm}

\begin{figure*}[h]
\vspace{-3mm}
\begin{center}
\subfloat[Initial graph.] {
\includegraphics[scale=0.25]{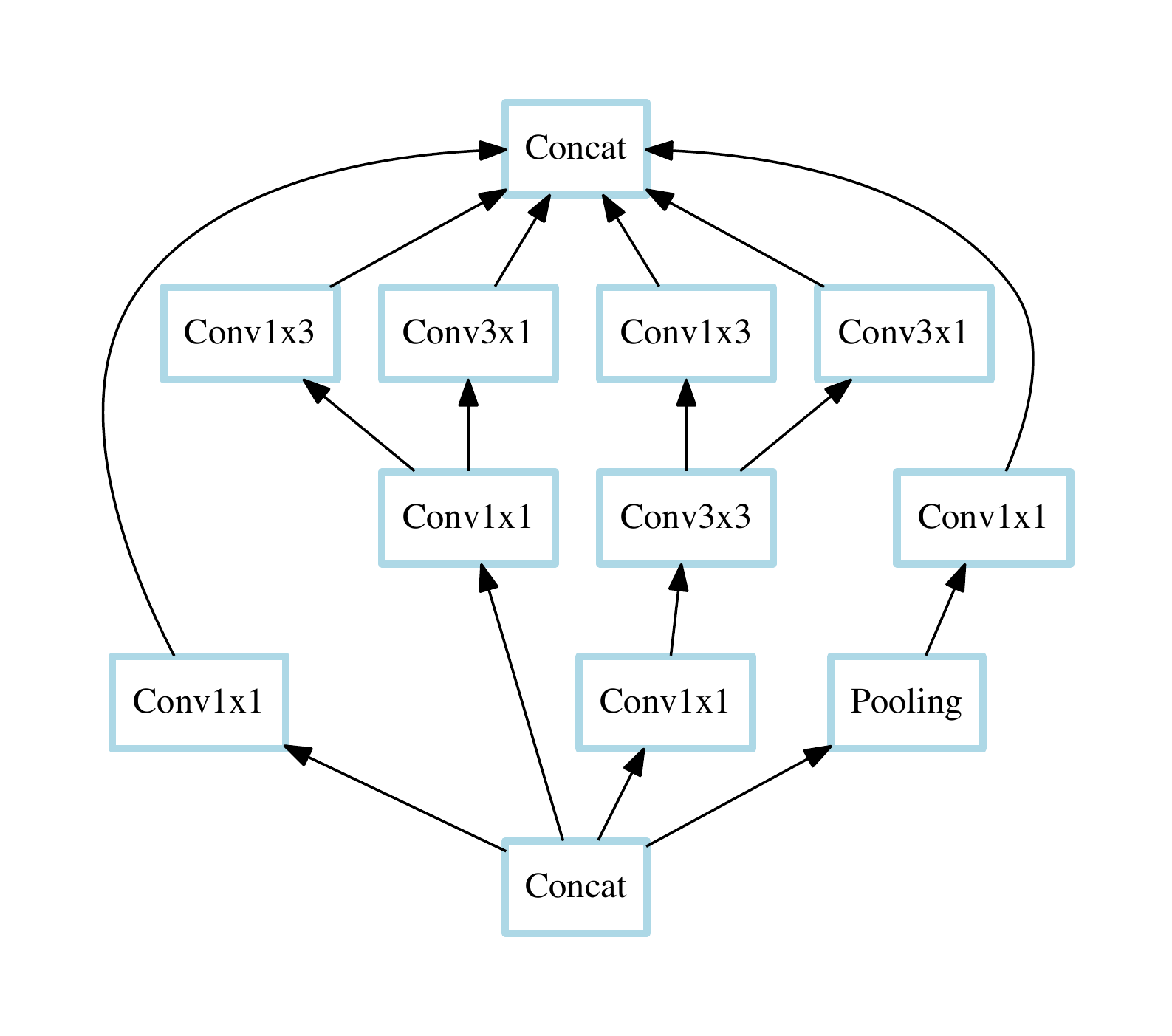}
}
\subfloat[Node elimination.] {
\includegraphics[scale=0.25]{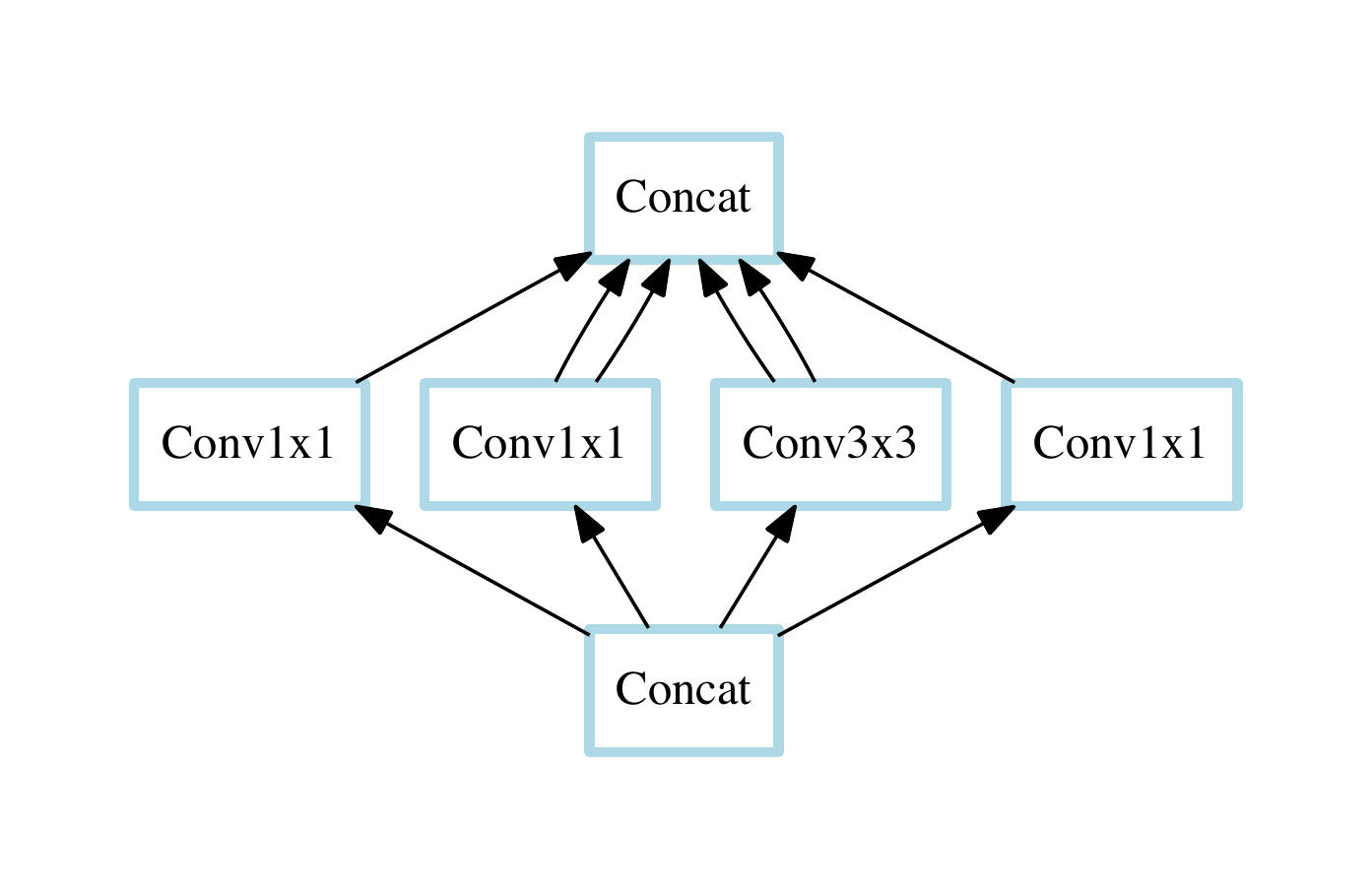}
}
\subfloat[Edge elimination.] {
\includegraphics[scale=0.25]{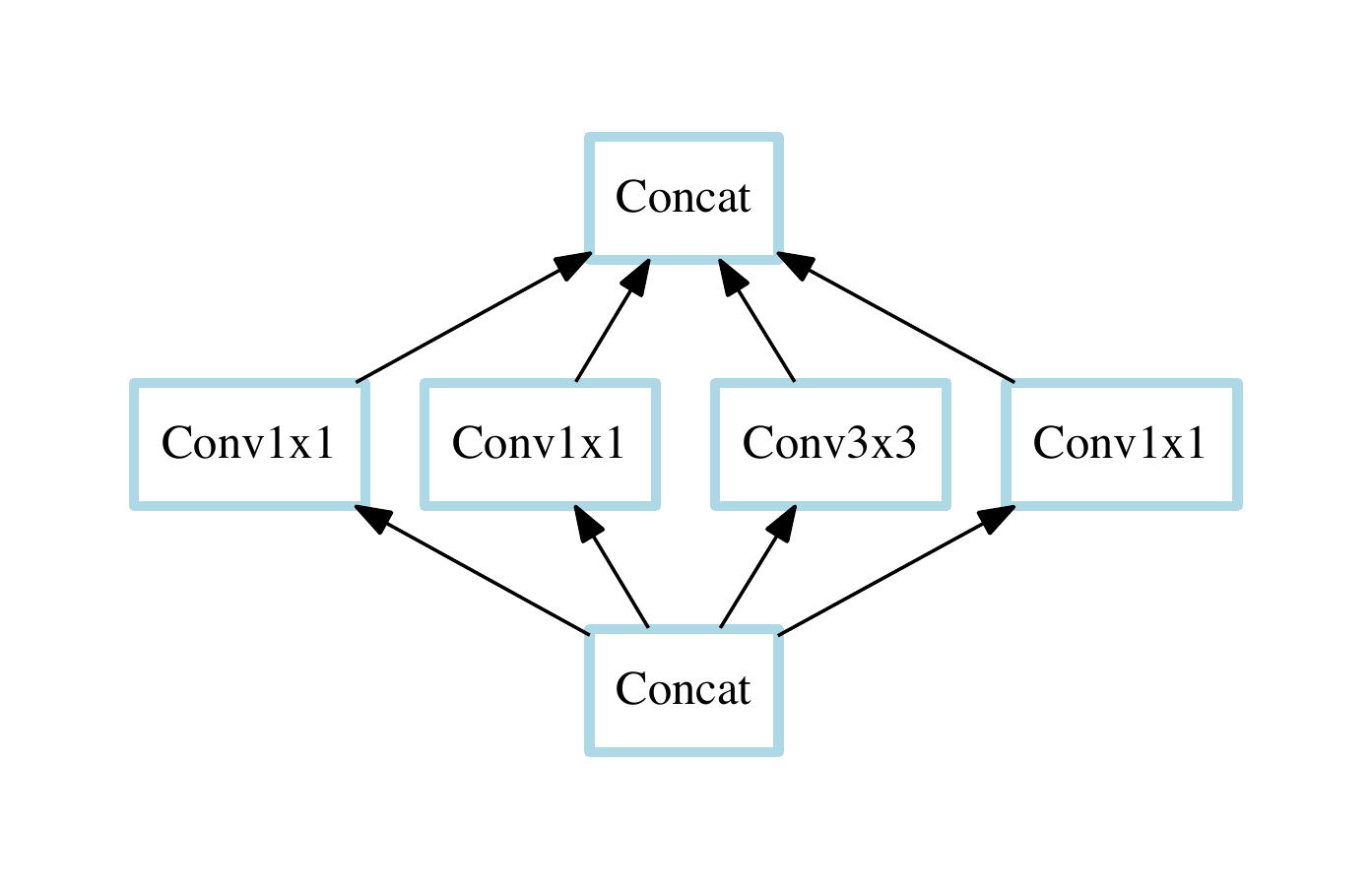}
}
\subfloat[Node elimination.] {
\includegraphics[scale=0.54]{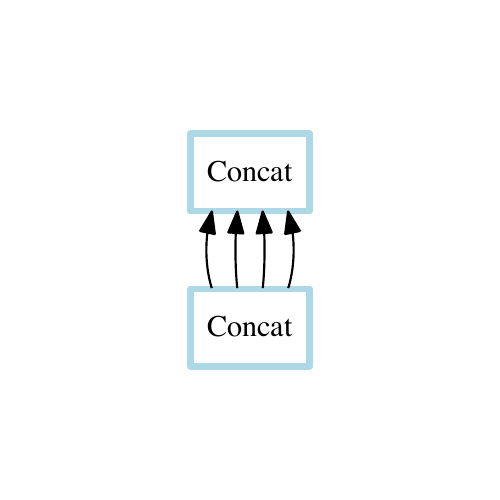}
}
\subfloat[Edge elimination.] {
\includegraphics[scale=0.54]{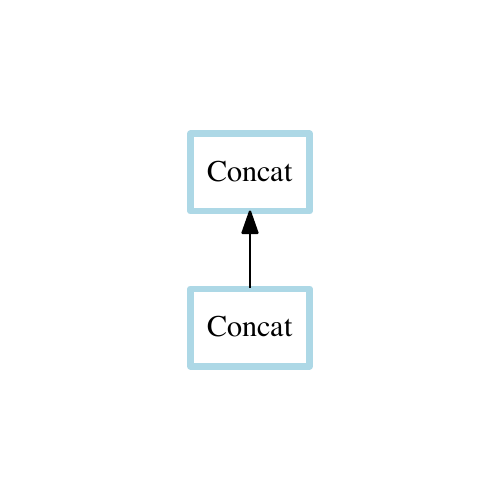}
}
\caption{Iteratively performing node/edge eliminations on an Inception module.}
\label{fig:reduce_example}
\end{center}
\vspace{-4mm}
\end{figure*}

We formally define node and edge eliminations and prove Theorem~\ref{thm1} and~\ref{thm2} in the appendix.\footnote{An extended version of this paper with proofs is available at https://arxiv.org/abs/1802.04924.}
The two theorems show that given an optimal parallelization strategy for the modified
graph, we can easily construct an optimal strategy for the original graph.

Algorithm~\ref{alg1} shows pseudocode using node and edge eliminations as subroutines to find an optimal parallelization strategy under our cost model.
The algorithm first iteratively uses node and edge eliminations to simplify an input computation graph until neither elimination can be applied (lines 4-13). 
Figure~\ref{fig:reduce_example} demonstrates how node and edge eliminations are performed on an Inception module~\cite{Inception}.

After the elimination phase, the algorithm enumerates all potential strategies for the final graph $\mathcal{G}^{(m)}$ and chooses 
$\mathcal{S}^{(m)}$ that minimizes $t_{\er o}(\mathcal{G}^{(m)}, \mathcal{D}, \mathcal{S}^{(m)})$ (line 14).
After deciding the configuration for each node in $\mathcal{G}^{(m)}$, we then decide the configurations for the eliminated nodes by iteratively undoing the eliminations in reverse order (lines 15-23). Theorem~\ref{thm1} and~\ref{thm2} guarantee that $\mathcal{S}^{(i)}$ is an optimal strategy for $\mathcal{G}^{(i)}$ ($0\leq i \leq m$). Finally, $\mathcal{S}^{(0)}$ is an optimal parallelization strategy for the original graph $\mathcal{G}$.

\begin{table}[t]
\vspace{-3mm}
\caption{Time complexity of Algorithm~\ref{alg1}. $C$ is the maximum number of potential configurations for a layer. $N$ and $E$ are the number of nodes and edges in $\mathcal{G}$, respectively. $K$ is the number of nodes in the final graph after node and edge eliminations.}
\label{tab:complexity}
\centering
\resizebox{\columnwidth}{!}{
\begin{tabular}{|l|l|}
\hline
{\bf Step} & {\bf Time Complexity}\\
\hline
Performing node and edge eliminations & $O(E C^3)$ \\
\hline
Finding the optimal strategy for the & \multirow{2}{*}{$O(K  C^K)$}\\
final graph  & \\
\hline
Undoing node and edge eliminations & $O(E  C)$\\
\hline
\hline
Overall  & $O(E  C^3 + K  C^K)$ \\
\hline
\end{tabular}
}
\vspace{-3mm}
\end{table}

{\bf Time complexity.} Table~\ref{tab:complexity} shows the time complexity of Algorithm~\ref{alg1}. Performing a node or edge elimination requires computing Equation~\ref{eqn2} or~\ref{eqn3} for the inserted edge, which takes $O(C^3)$ and $O(C^2)$ time, respectively.
The total number of node and edge eliminations is smaller than $E$, since an elimination reduces the number of edges in the graph by one. Therefore, the time complexity for performing and undoing node and edge eliminations is $O(EC^3)$ and $O(EC)$, respectively.
The algorithm enumerates all possible strategies for the final graph $\mathcal{G}^{(m)}$, which takes $O(KC^K)$ time.
The algorithm works efficiently on a wide range of real-world CNNs including \alexnet~\cite{AlexNet}, VGG~\cite{VGG}, \inception~\cite{Inception}, and ResNet~\cite{Resnet}, all of which are reduced to a final graph with only 2 nodes (i.e., $K=2$).

\begin{table}[t]
\caption{Execution time for finding the optimal parallelization strategy for 4 GPUs. 
Note that the number of nodes in the final graph (i.e., $K$) is equal to 2 for all networks.
For LeNet-5 and \alexnet, two algorithms find the same optimal strategies.}
\label{tab:algorithm_runtimes}
\centering
\resizebox{\columnwidth}{!}{
\begin{tabular}{|r|c|c|c|}
\hline
{\bf Network} & {\bf \# Layers} & {\bf Baseline} & {\bf Our Algorithm} \\
\hline
LeNet-5 & 6 & 5.6 seconds & 0.01 seconds \\
\alexnet & 11 & 2.1 hours & 0.02 seconds \\
\vgg & 21 &$>$ 24 hours & 0.1 seconds \\
\inception & 102 & $>$ 24 hours & 0.4 seconds \\
\hline
\hline
\multicolumn{2}{|c|}{Time complexity} & $O(EC^N)$ & $O(EC^3)$ \\
\hline
\end{tabular}
}
\vspace{-3mm}
\end{table}

We compare Algorithm~\ref{alg1} with a baseline algorithm that uses a {\em depth-first search} algorithm to find an optimal strategy
for the original graph $\mathcal{G}$. Table~\ref{tab:algorithm_runtimes} compares the time complexity and actual execution time of the two algorithms.
Our algorithm achieves lower time complexity and reduces the execution time by orders of magnitude over the baseline.

\section{Experiments}
\label{sec:exp}
We found that it is non-trivial to parallelize a layer in the height, width, and channel dimensions in existing frameworks (e.g., TensorFlow, PyTorch, and Caffe2), and none provides an interface for controlling parallelization at the granularity of individual layers. 
Therefore, we implemented our framework in Legion~\cite{Legion12}, a high-performance parallel runtime for distributed heterogeneous architectures, and use cuDNN~\cite{cudnn} and cuBLAS~\cite{cublas} as the underlying libraries to process neural network layers. 
The following Legion features significantly simplify our implementation.
First, Legion supports high-dimensional partitioning that allows us to parallelizing any layer in any combination of the dimensions. 
Second, Legion permits control of parallelization at the granularity of each layer.
Third, Legion allows fine-grain control over the placement of tasks and data.
Fourth, the underlying implementation of Legion automatically and systematically overlaps communication with computation and optimizes the path and pipelining of data movement across the machine~\cite{Realm14, DMA, Lux}.


{\bf Benchmarks.} We evaluate our approach on three established CNNs.
\alexnet~\cite{AlexNet} is the winner of the ILSVRC-2012 image classification competition.
\vgg~\cite{VGG} improves network accuracy by pushing the depth of the network to 16 weighted layers.
\inception~\cite{Inception} is a 102-layer deep CNN that uses carefully designed Inception modules to increase the number of layers while maintaining a reasonable computational budget.

{\bf Datasets.} We evaluate the runtime performance of all three CNNs on the ImageNet-1K dataset~\cite{imagenet2009} that consists of 1.2 million images from 1,000 categories. 

{\bf Baselines.} We compare the following parallelization strategies in the experiments.
\begin{itemize}
\item[1.] {\bf Data parallelism} is the most common parallelization strategy for large-scale training~\cite{LargeSGD, Tensorflow}. In data parallelism, each device has a copy of the entire network and processes a subset of the training dataset.
\item[2.] {\bf Model parallelism.} We use a model parallelism approach~\cite{OWT} as a baseline, which distributes the network parameters in each layer equally to all devices, providing good load balancing.
\item[3.] {\bf OWT parallelism} is designed to reduce communication costs by using data parallelism for convolutional and pooling layers and switching to model parallelism for densely-connected layers.
\item[4.] {\bf Layer-wise parallelism.} Given a computation graph and a set of available devices, we run the algorithm described in Section~\ref{sec:method} to find a parallelization strategy minimizing Equation~\ref{eqn1}.
\end{itemize}

{\bf Experimental setup.} All experiments were performed on a GPU cluster with 4 compute nodes, each of which is equipped with two Intel 10-core E5-2600 CPUs, 256G main memory, and four NVIDIA Tesla P100 GPUs. GPUs on the same node are connected by NVLink, and nodes are connected over 100Gb/s EDR Infiniband.
We use synchronous training and a per-GPU batch size of 32 for all experiments.

To rule out implementation differences, we ran data parallelism experiments in TensorFlow r1.7, PyTorch v0.3, and our implementation and compared the runtime performance. 
Our Legion-based framework achieves the same or better runtime performance on all three CNNs compared to TensorFlow and PyTorch, and therefore we report the data parallelism performance achieved by our framework. 

\subsection{Runtime Performance}
\label{subsec:eval_results}
\begin{figure*}[t!]
\begin{center}
\includegraphics[scale=0.36]{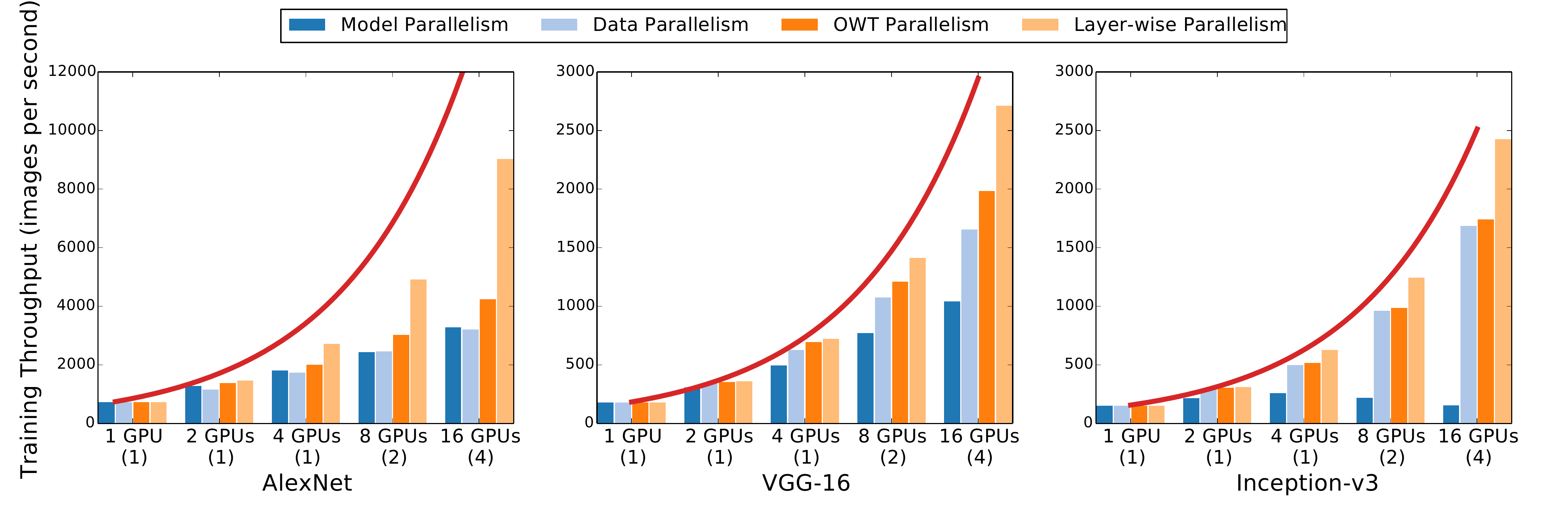}
\vspace{-5mm}
\caption{Training throughput (i.e., number of images processed per second) with different parallelization strategies (higher is better). 
Numbers in parenthesis are the number of compute nodes used in the experiments. 
The red lines show the training throughput in linear scale (ideal case).}
\label{fig:train_tp}
\end{center}
\end{figure*}

\begin{figure*}[t!]
\begin{center}
\includegraphics[scale=0.36]{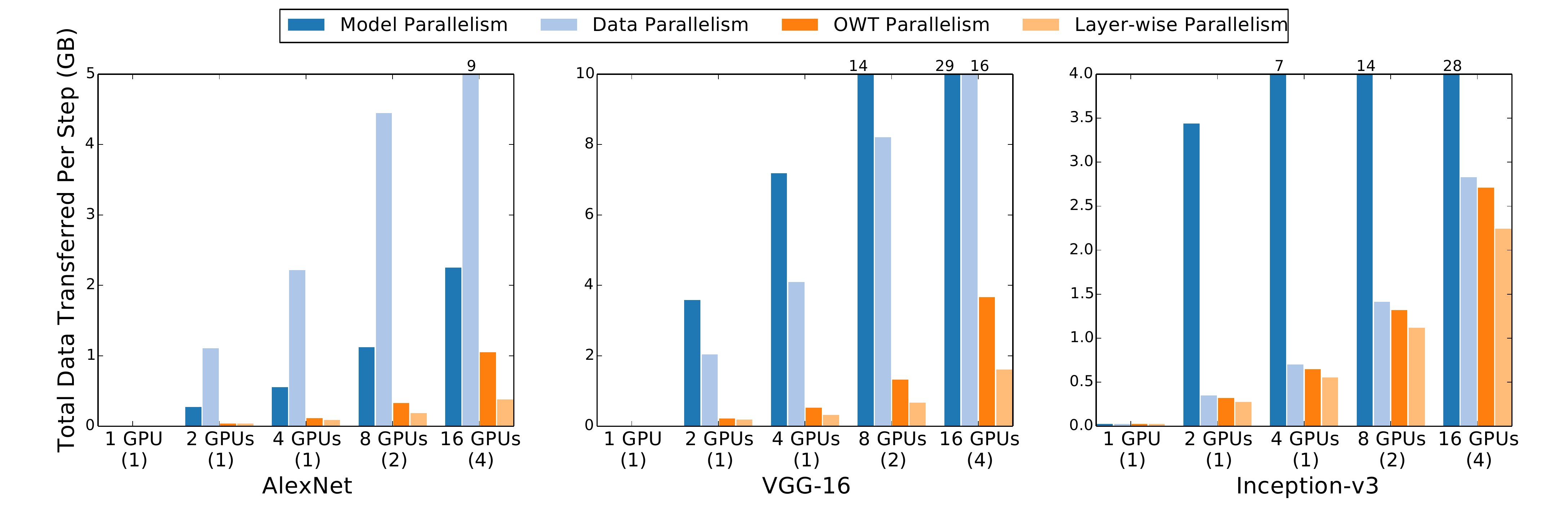}
\vspace{-5mm}
\caption{Communication cost (i.e., data transferred in each step) with different parallelization strategies (lower is better).}
\label{fig:transfer_overhead}
\end{center}
\vspace{-3mm}
\end{figure*}

We compare the training throughput and communication cost among different parallelization strategies.
Figure~\ref{fig:train_tp} shows the training throughputs with different CNNs and different sets of available devices.
Both model and data parallelism scale well on a single node but show limited scalability in distributed training, where the inter-node communication limits runtime performance.
OWT parallelism achieves improved runtime performance by switching to model parallelism for densely-connected layers to reduce communication costs.
In all experiments, layer-wise parallelism consistently outperforms the other strategies and increases the training throughput by up to 2.2$\times$, 1.5$\times$, and 1.4$\times$ for \alexnet, \vgg, and \inception, respectively.

In addition, layer-wise parallelism achieves better scalability than the other strategies. 
Scaling the training of the three CNNs from 1 GPU to 16 GPUs (on 4 nodes), layer-wise parallelism achieves 12.2$\times$, 14.8$\times$, and 15.5$\times$ speedup for \alexnet, \vgg, and \inception, respectively, while the best other strategy achieves 6.1$\times$, 10.2$\times$, and 11.2$\times$ speedup.
Moreover, Figure~\ref{fig:train_tp} shows that the layer-wise parallelism can help bridge the runtime performance gap between the ideal training throughputs in linear scale (the red lines) and the actual training throughputs achieved by current parallelization strategies.
This shows that layer-wise parallelism is more efficient for large-scale training.


Communication cost is another important performance metric in large-scale training.
Figure~\ref{fig:transfer_overhead} compares the communication costs of different strategies.
OWT parallelism eliminates gradient synchronization for densely-connected layers and reduces overall communication costs by 1.1-23.0$\times$ compared to data and model parallelism.
In addition, Layer-wise parallelism outperforms OWT parallelism by further reducing communication overhead by 1.2-2.5$\times$.

\subsection{Cost Model}
\label{subsec:eval_accuracy}
\begin{table}[t]
\vspace{-3mm}
\centering
\caption{Relative difference between estimated execution time $t_{\er o}(\mathcal{G}, \mathcal{C})$ and actual execution time $t(\mathcal{G}, \mathcal{C})$.}
\resizebox{0.95\columnwidth}{!}{
\begin{tabular}{|r|c|c|c|}
\hline
\multirow{2}{*}{\bf Available Devices} & \multicolumn{3}{c|}{$(t_{\er o}(\mathcal{G}, \mathcal{C}) - t(\mathcal{G}, \mathcal{C})) / t(\mathcal{G}, \mathcal{C})$} \\
\cline{2-4}
 & {\bf \alexnet} & {\bf \vgg} & {\bf \inception} \\
\hline
1 GPU (1 node) & 1\% & 0\% & 1\%\\
2 GPUs (1 node) & 4\% & 3\% & 5\%\\
4 GPUs (1 node) & {-5\%} & 2\% & 5\%\\
8 GPUs (2 nodes) & 2\% & 6\% & {9\%}\\
16 GPUs (4 nodes) & -1\% & {7\%} & 6\%\\
\hline
\end{tabular}
}
\vspace{-3mm}
\label{tab:accuracy}
\end{table}

We compare the estimated execution time $t_{\er o}(\mathcal{G}, \mathcal{D}, \mathcal{C})$ projected by our cost model (see Equation~\ref{eqn1}) with the measured per-step execution time $t(\mathcal{G}, \mathcal{D}, \mathcal{C})$ in the experiments. The results are shown in Table~\ref{tab:accuracy}. In all experiments, the relative difference between the estimated
and the real execution time is within 10\%, showing that the cost model can reliably predict a CNN's per-step execution time given the set of available devices and the connections between them. 

\subsection{Analysis of Optimal Parallelization Strategies}
\label{subsec:eval_analysis}
\begin{table}[t]
\centering
\caption{An optimal parallelization strategy under the cost model for parallelizing \vgg on 4 GPUs on a single compute node.}
\resizebox{0.9\columnwidth}{!}{
\begin{tabular}{|c|c|}
\hline
{\bf Layers} & {\bf Parallelization Configuration}\\
\hline
2 x Conv + Pooling & \multirow{4}{*}{\{n=4, h=1, w=1, c=1\}} \\
2 x Conv + Pooling & \\
3 x Conv + Pooling & \\
3 x Conv + Pooling & \\
\hline
3 x Conv + Pooling & \{n=1, h=2, w=2, c=1\} \\
\hline
Fully-connected & \multirow{2}{*}{\{n=1, c=4\}}\\ 
Fully-connected & \\
\hline
Fully-connected & \{n=1, c=2\} \\
\hline
Softmax & \{n=1, c=1\} \\
\hline
\end{tabular}
}
\label{tab:vgg_config}
\vspace{-3mm}
\end{table}

We analyze the optimal parallelization strategies under our cost model and find several similarities among them.

First, for the beginning layers of a CNN with large height/width dimensions and a small channel dimension, an optimal strategy usually uses data parallelism on all available devices, since the communication costs for synchronizing gradients are much smaller than the communication costs for moving tensors between layers.

Second, deeper layers in a CNN tend to have smaller height/width dimensions and a larger channel dimension.
As a result, the costs for moving tensors between different layers decrease, while the costs for synchronizing parameters increase.
An optimal strategy adaptively reduces the number of devices for these layers to reduce communication costs to synchronize parameters and
opportunistically uses parallelism in the height/width dimensions to achieve better runtime performance.

Finally, for densely-connected layers, an optimal strategy eventually switches to model parallelism on a small number of devices, because 
synchronizing gradients and transferring tensors are both much more expensive than the execution time for densely-connected layers.
This reduces the communication costs for synchronizing parameters and moving tensors at the cost of only using a subset of available devices.

Table~\ref{tab:vgg_config} shows an optimal strategy under the cost model for parallelizing \vgg on 4 GPUs.
This strategy first uses parallelism in the sample dimension for the beginning convolutional and pooling layers and then uses parallelism in both the height and width dimensions to accelerate the last three convolutional layers. 
For the fully-connected layers, it uses parallelism in the channel dimension to reduce communication costs and adaptively decreases the degrees of parallelism.

\section{Conclusion}
We have introduced layer-wise parallelism, which allows each layer in a CNN to use an individual parallelization configuration. 
We propose a cost model that quantitively evaluates the runtime performance of different strategies and use a dynamic programming based graph search algorithm to find a globally optimal strategy under the cost model.
Our experiments show that layer-wise parallelism significantly outperforms state-of-the-art strategies for CNNs by increasing training throughput, reducing communication costs, and achieving better scalability on larger numbers of devices.


\section*{Acknowledgements}
This research was supported by NSF grant CCF-1160904 and the Exascale Computing Project (17-SC-20-SC), a collaborative effort of the U.S. Department of Energy Office of Science and the National Nuclear Security Administration.

\bibliography{bibliography}
\bibliographystyle{icml2018}
\appendix

\section{Node and Edge Eliminations}
We define node and edge eliminations in Algorithm~\ref{alg2}.
\begin{algorithm}
\caption{Node and edge eliminations.}
\label{alg2}
\begin{algorithmic}[1]
\Function{NodeElimination}{$\mathcal{G}$}
\If {there exists a node $l_j$ with a single in-edge $e_1=(l_i, l_j)$ and a single out-edge $e_2=(l_j, l_k)$ }
\State $e' = (l_i, l_k)$
\State $\mathcal{G'} = \mathcal{G} - l_j - e_1 - e_2 + e'$
\State {\bf return} $\mathcal{G'}$
\Else
\State {\bf return} $\mathcal{G}$
\EndIf
\EndFunction
\State
\Function{EdgeElimination}{$\mathcal{G}$}
\If {there exist two edges $e_1 = (l_i, l_j)$ and $e_2 = (l_i, l_j)$}
\State $e' = (l_i, l_j)$
\State $\mathcal{G'} = \mathcal{G} - e_1 - e_2 + e'$
\State {\bf return} $\mathcal{G'}$
\Else
\State {\bf return} $\mathcal{G}$
\EndIf
\EndFunction
\State
\end{algorithmic}
\end{algorithm}

\begin{theorem}
\label{thm3}
Assume $\mathcal{G}'$ = NodeElimination($\mathcal{G}$) and $l_j$ is the eliminated layer. If ${\mathcal{S}_o}'$ is an optimal strategy for $\mathcal{G}'$,
then $\mathcal{S}_o = {\mathcal{S}_o}' + \widehat{c_j}$ is an optimal strategy for $\mathcal{G}$, where
\begin{equation}
\label{eqn_ap1}
\begin{split}
\widehat{c_j} = \argmin_{c_j} & \{ t_{\er c}(n_j, c_j) + t_{\er s}(n_j, c_j) \\
& + t_{\er x}(e_1, c_i, c_j) + t_{\er x}(e_2, c_j, c_k)\}
\end{split}
\end{equation}
\end{theorem}

\begin{proof}
It suffices to prove that $t_{\er o}(\mathcal{G}, \mathcal{S}_1) \geq t_{\er o}(\mathcal{G}, \mathcal{S}_o)$
for any other strategy $\mathcal{S}_1$. We assume layer $l_i$ has parallelization configuration $c_{i1}\in \mathcal{S}_1$.
We claim that
\begin{eqnarray}
t_{\er o}(\mathcal{G}, \mathcal{S}_1) &  \geq t_{\er o}(\mathcal{G'}, \mathcal{S}_1)  \label{eqn_ap2} \\
&  \geq t_{\er o}(\mathcal{G'}, {\mathcal{S}_o}') \label{eqn_ap3}\\
& = t_{\er o}(\mathcal{G}, \mathcal{S}_o) \label{eqn_ap4}
\end{eqnarray}

To prove (\ref{eqn_ap2}), note that the difference between $t_{\er o}(\mathcal{G}, \mathcal{S}_1)$ and  $t_{\er o}(\mathcal{G'}, \mathcal{S}_1)$ is
\begin{equation}
\label{eqn_a5}
\begin{split}
& t_{\er o}(\mathcal{G}, \mathcal{S}_1) - t_{\er o}(\mathcal{G'}, \mathcal{S}_1) \\
= & t_{\er c}(l_j, c_{j1}) + t_{\er s}(l_j, c_{j1}) +  t_{\er x}(e_1, c_{i1}, c_{j1})  \\
& + t_{\er x}(e_2, c_{j1}, c_{k1}) - t_{\er x}(e', c_{i1}, c_{k1}) \\
\end{split}
\end{equation}
because all other layers except $l_j$ use the same configurations in $t_{\er o}(\mathcal{G}, \mathcal{S}_1)$ and $t_{\er o}(\mathcal{G'}, \mathcal{S}_1)$, and
therefore all cost functions unrelated to $l_j$ drop out in the subtraction. The remaining parts are $l_j$, $e_1$, and $e_2$, which no longer exist in $\mathcal{G'}$ after node
elimination, and $e'$ that is added to $\mathcal{G'}$. Recall that $t_{\er x}(e', \cdot, \cdot)$ is defined as follows.
\begin{equation}
\begin{split}
\label{eqn_a6}
t_{\er x}(e', c_i, c_k) = \min_{c_j} & \{ t_{\er c}(l_j, c_j) + t_{\er s}(l_j, c_j) \\
& + t_{\er x}(e_1, c_i, c_j) + t_{\er x}(e_2, c_j, c_k)\}
\end{split}
\end{equation}
Combining (\ref{eqn_a5}) and (\ref{eqn_a6}), we have $t_{\er o}(\mathcal{G}, \mathcal{S}_1) \geq t_{\er o}(\mathcal{G'}, \mathcal{S}_1)$.

To prove (\ref{eqn_ap3}), simply observe that the inequality must hold because ${\mathcal{S}_o}'$ is assumed to be an optimal strategy for $\mathcal{G'}$.

To prove (\ref{eqn_ap4}), the difference between $t_{\er o}(\mathcal{G'}, {\mathcal{S}_o}')$ and $t_{\er o}(\mathcal{G}, \mathcal{S}_o)$ is
\begin{equation}
\label{eqn_ap7}
\begin{split}
& t_{\er o}(\mathcal{G}, \mathcal{S}_o) - t_{\er o}(\mathcal{G'}, {\mathcal{S}_o}') \\
= & t_{\er c}(l_j, \widehat{c_j}) + t_{\er s}(l_j, \widehat{c_j}) +  t_{\er x}(e_1, c_{i}, \widehat{c_j}) \\
&+ t_{\er x}(e_2, \widehat{c_j}, c_{k}) - t_{\er x}(e', c_{i}, c_{k}) \\
\end{split}
\end{equation}
This is because $\mathcal{S}_o = {\mathcal{S}_o}' + \widehat{c_j}$, and therefore all cost functions unrelated to $l_j$ drop out. 
We can prove (\ref{eqn_ap4}) by plugging (\ref{eqn_ap1}) into (\ref{eqn_ap7}).
\end{proof}

\begin{theorem}
\label{thm4}
Assume $\mathcal{G}'$ = EdgeElimination($\mathcal{G}$), and ${\mathcal{S}_o}'$ is an optimal strategy for $\mathcal{G}'$,
then $\mathcal{S}_o={\mathcal{S}_o}'$ is an optimal strategy for $\mathcal{G}$.
\end{theorem}

\begin{proof}
The proof is the same sequence of steps for Theorem~\ref{thm3}, but the justification of each step is different.

To prove (\ref{eqn_ap2}) for Theorem~\ref{thm4}, the difference between $t_{\er o}(\mathcal{G}, \mathcal{S}_1)$ and  $t_{\er o}(\mathcal{G'}, \mathcal{S}_1)$ is
\begin{equation}
\label{eqn_ap8}
\begin{split}
& t_{\er o}(\mathcal{G}, \mathcal{S}_1) - t_{\er o}(\mathcal{G'}, \mathcal{S}_1) \\
= &  t_{\er x}(e_1, c_{i1}, c_{j1}) + t_{\er x}(e_2, c_{i1}, c_{j1}) - t_{\er x}(e', c_{i1}, c_{j1})
\end{split}
\end{equation}
Recall that $t_{\er x}(e', \cdot, \cdot)$ is defined as follows:
\begin{equation}
\label{eqn_ap9}
t_{\er x}(e', c_i, c_j) = t_{\er x}(e_1, c_i, c_j) + t_{\er x}(e_2, c_i, c_j)
\end{equation}
Combining (\ref{eqn_ap8}) and (\ref{eqn_ap9}), we have $t_{\er o}(\mathcal{G}, \mathcal{S}_1) = t_{\er o}(\mathcal{G'}, \mathcal{S}_1)$.

For (\ref{eqn_ap3}), the inequality holds because ${\mathcal{S}_o}'$ is an optimal strategy for $\mathcal{G'}$.

For (\ref{eqn_ap4}), the difference between $t_{\er o}(\mathcal{G'}, {\mathcal{S}_o}')$ and $t_{\er o}(\mathcal{G}, \mathcal{S}_o)$ is
\begin{equation}
\label{eqn_ap10}
\begin{split}
& t_{\er o}(\mathcal{G}, \mathcal{S}_o) - t_{\er o}(\mathcal{G'}, {\mathcal{S}_o}') \\
= & t_{\er x}(e_1, c_{i}, c_j) + t_{\er x}(e_2, c_i, c_j) - t_{\er x}(e', c_{i}, c_j) \\
= & 0
\end{split}
\end{equation}

\end{proof}

\end{document}